\documentclass{article}
\usepackage[preprint]{colm2026_conference}

\usepackage{microtype}

\usepackage{amsfonts}
\usepackage{amsmath}
\usepackage{amssymb}
\usepackage{nicefrac}

\newcounter{proposition}
\renewcommand{\theproposition}{\arabic{proposition}}
\newenvironment{proposition}{\refstepcounter{proposition}\par\medskip\noindent\textbf{Proposition~\theproposition.}\itshape}{\par\medskip}
\newenvironment{proof}{\par\noindent\textbf{Proof.}}{\hfill$\square$\par\medskip}

\usepackage{booktabs}
\usepackage{graphicx}
\usepackage[dvipsnames]{xcolor}
\usepackage{caption}
\usepackage{colortbl}
\usepackage{wrapfig}

\definecolor{mygray}{gray}{.9}

\usepackage{enumitem}

\usepackage{xspace}

\usepackage{algorithm}
\usepackage{algorithmic}

\usepackage{lineno}

\usepackage{url}
\usepackage[colorlinks=true,citecolor=DarkOrchid,linkcolor=NavyBlue,urlcolor=NavyBlue]{hyperref}

\usepackage{amsmath,amsfonts,bm}

\def\eqref#1{equation~\ref{#1}}

\def\1{\bm{1}}

\DeclareMathAlphabet{\mathsfit}{\encodingdefault}{\sfdefault}{m}{sl}
\SetMathAlphabet{\mathsfit}{bold}{\encodingdefault}{\sfdefault}{bx}{n}

\newcommand{\textproc}[1]{\textsc{#1}}

\title{IndexCache: Accelerating Sparse Attention via Cross-Layer Index Reuse}

\author{Yushi Bai$^{1\dagger}$, Qian Dong$^{1\dagger}$, Ting Jiang$^{2}$, Xin Lv$^{2}$ \\ \textbf{Zhengxiao Du$^{2}$, Aohan Zeng$^{12}$, Jie Tang$^{1}$, Juanzi Li$^{1}$}\vspace{0.42em} \\ 
$^1$Tsinghua University
  \quad
  $^2$Z.ai
}

\footnotetext[2]{Work done while interned at Z.ai.}

\begin{document}

\ifcolmsubmission
\linenumbers
\fi

\maketitle

\begin{abstract}
Long-context agentic workflows have emerged as a defining use case for large language models, making attention efficiency critical for both inference speed and serving cost.
Sparse attention addresses this challenge effectively, and DeepSeek Sparse Attention~(DSA) is a representative production-grade solution: a lightweight \emph{lightning indexer} selects the top-$k$ most relevant tokens per query, reducing core attention from~$O(L^2)$ to~$O(Lk)$.
However, the indexer itself retains~$O(L^2)$ complexity and must run independently at every layer, despite the fact that the resulting top-$k$ selections are highly similar across consecutive layers.
We present \textbf{IndexCache}, which exploits this cross-layer redundancy by partitioning layers into a small set of \emph{Full}~layers that run their own indexers and a majority of \emph{Shared}~layers that simply reuse the nearest Full layer's top-$k$ indices.
We propose two complementary approaches to determine and optimize this configuration.
\emph{Training-free IndexCache} applies a greedy search algorithm that selects which layers to retain indexers by directly minimizing language modeling loss on a calibration set, requiring no weight updates.
\emph{Training-aware IndexCache} introduces a multi-layer distillation loss that trains each retained indexer against the averaged attention distributions of all layers it serves, enabling even simple interleaved patterns to match full-indexer accuracy.
Experimental results on a 30B DSA model show that IndexCache can remove 75\%\ of indexer computations with negligible quality degradation, achieving up to 1.82$\times$~prefill speedup and 1.48$\times$~decode speedup compared to standard DSA.
These positive results are further confirmed by our preliminary experiments on the production-scale GLM-5 model (Figure~\ref{fig:glm5}).
\end{abstract}

\begin{figure}[htbp]
\centering
\includegraphics[width=\textwidth]{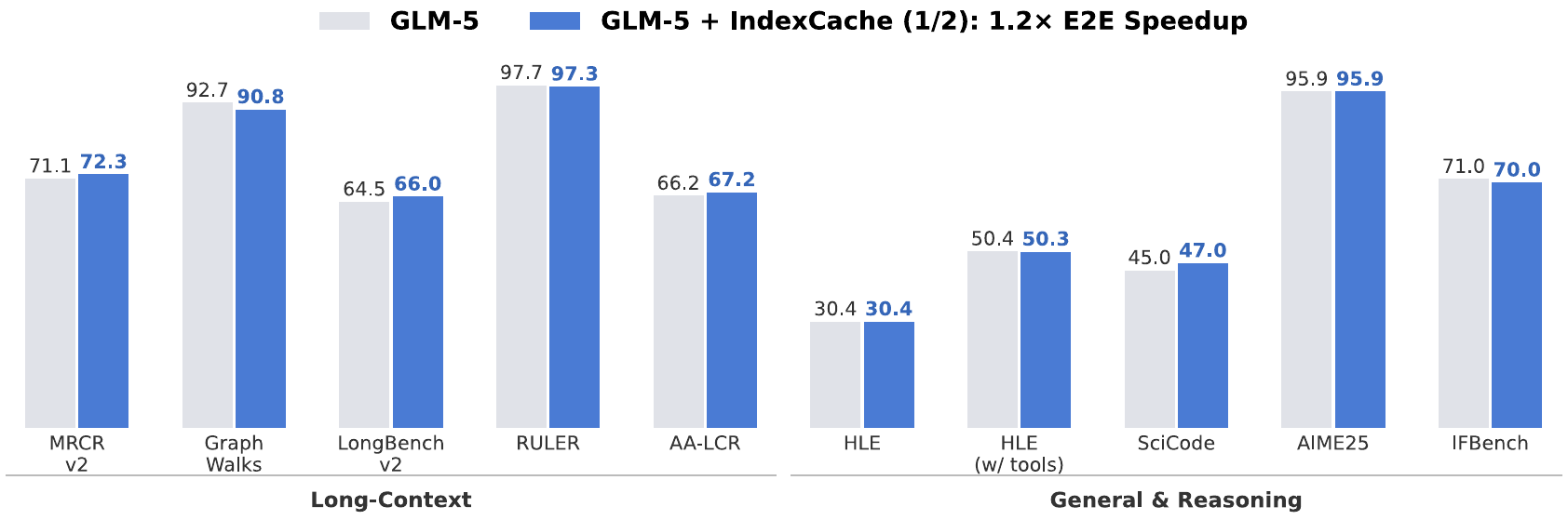}
\caption{%
    Benchmark comparison between GLM-5 and GLM-5 + IndexCache. IndexCache removes 50\% of indexer computations while maintaining comparable performance across both long-context and reasoning tasks, delivering $\sim1.2\times$ end-to-end speedup.
}
\label{fig:glm5}
\end{figure}

\section{Introduction}
\label{sec:intro}

The self-attention mechanism~\citep{vaswani2017attention} is a cornerstone of modern large language models~(LLMs), yet its quadratic complexity in sequence length presents a fundamental bottleneck for long-context inference.
As LLMs are increasingly deployed in settings that demand extended contexts, such as long chain-of-thought reasoning, multi-step agentic workflows, and retrieval-augmented generation over web-scale sources, reducing attention cost without sacrificing model quality has become a critical research problem.

\begin{wrapfigure}{r}{0.4\linewidth}
    \centering
    \vspace{-1em}
    \includegraphics[width=\linewidth]{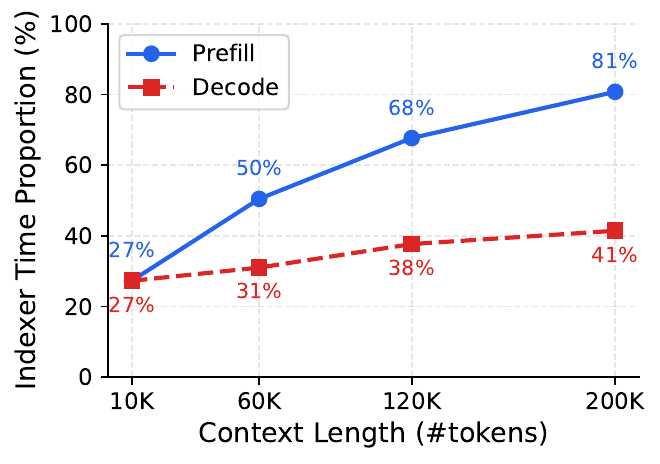}
    \vspace{-2em}
\end{wrapfigure}
Sparse attention offers a principled solution: instead of attending to all preceding tokens, each query selects only the most relevant subset.
Among recent approaches~\citep{nsa,moba,infllm-v2,hysparse,longcat}, DeepSeek Sparse Attention~(DSA)~\citep{deepseekv32} stands out as a production-grade trainable sparse attention mechanism.
For sparse token selection, DSA introduces an additional \emph{lightning indexer} module at each layer that scores all preceding tokens and selects the top-$k$ for the subsequent core attention.
This reduces per-layer core attention from~$O(L^2)$ to~$O(Lk)$ while preserving model quality through continued pre-training.
However, the indexer itself still operates at~$O(L^2)$ at \emph{every} layer: although cheaper per-FLOP than the main attention computation, its total cost across~$N$ layers is~$O(NL^2)$, which grows quadratically with context length and becomes a non-negligible fraction of the total attention budget.
As shown in the figure above, profiling a 30B DSA model reveals that the indexer's share of total latency rises sharply with context length, particularly during the prefill stage, while the rest of the computations grow only modestly. This indicates that \textbf{reducing indexer cost is the key to accelerating long-context DSA inference}.

A key insight motivating our work is that the top-$k$ selections produced by the indexer are \emph{highly correlated across consecutive layers}---an instance of the broader cross-layer token selection stability observed in full attention models~\citep{kascade,hysparse}.
While prior methods exploit this stability by reusing indices from full attention anchor layers, they do not directly apply to sparse attention, since in DSA, full attention is only computed via the lightweight indexer.
We empirically verify this for DSA by computing the pairwise top-$k$ index overlap across all layers (Appendix~\ref{app:topk-overlap}): adjacent layers share 70-100\% of their selected tokens, and the heatmap reveals distinct layer clusters with mutually high overlap, suggesting that most indexer computations are redundant.
This leaves a simple but impactful opportunity: \emph{can we remove the majority of indexers in DSA and let most layers reuse top-$k$ indices from a small number of retained indexer layers, without degrading quality?}

We answer affirmatively with \textbf{IndexCache}, a method that eliminates up to~75\%\ of indexer computations in DSA through cross-layer index reuse.
IndexCache partitions layers into \texttt{F}~(Full) layers that retain their indexers and \texttt{S}~(Shared) layers that inherit top-$k$ indices from the nearest preceding~\texttt{F}~layer, adding only one conditional branch in inference~(Figure~\ref{fig:pseudocode-comparison}).
We propose two complementary approaches to determine and optimize this configuration:

\begin{itemize}[leftmargin=*,itemsep=4pt]
    \item \textbf{Training-free IndexCache} applies to any off-the-shelf DSA model without weight updates.
    We show that na\"{\i}ve uniform interleaving degrades quality, and propose a \emph{greedy layer selection} algorithm that uses LM loss on a small calibration set to determine the optimal pattern, namely which layers to retain indexers.
    The greedy solution retains only~$1/4$ of indexers while matching the original DSA model's downstream performance.

    \item \textbf{Training-aware IndexCache} optimizes the model \emph{parameters} for cross-layer sharing.
    We introduce a \emph{multi-layer distillation loss} that trains each retained indexer against the attention distributions of all layers it serves.
    Under this objective, even a simple uniform interleaved pattern achieves on par with the original per-layer indexer design.
\end{itemize}

Empirically, on a 30B DSA model evaluated across nine long-context and reasoning benchmarks, both training-free (with greedy pattern search) and training-aware IndexCache retain only~$1/4$ of indexers with negligible quality degradation, yielding up to \textbf{1.82$\times$}~prefill speedup and \textbf{1.48$\times$}~decode speedup at 200K context length. Preliminary results on the 744B GLM-5 further confirm scalability, achieving at least 1.3$\times$ speedup with negligible degradation in long-context performance.

\begin{figure}[t]
\centering
\begin{minipage}[t]{0.48\textwidth}
\begin{algorithm}[H]
\caption*{\textbf{(a) Standard DSA Inference}}
\begin{algorithmic}[1]
\REQUIRE Input~$\mathbf{X}$, layers~$1 \ldots N$
\FOR{$\ell = 1$ \TO $N$}
    \STATE $\mathbf{I}^{(\ell)} \leftarrow \textproc{Indexer}_\ell(\mathbf{X})$
    \STATE $\mathcal{T}^{(\ell)} \leftarrow \mathrm{Top\text{-}k}(\mathbf{I}^{(\ell)})$
    \STATE $\mathbf{X} \leftarrow \textproc{SparseAttn}_\ell(\mathbf{X},\, \mathcal{T}^{(\ell)})$
    \STATE $\mathbf{X} \leftarrow \textproc{FFN}_\ell(\mathbf{X})$ \hfill $\triangleright$~\textrm{\small + norm, residual, etc.}
\ENDFOR
\end{algorithmic}
\end{algorithm}
\end{minipage}
\hfill
\begin{minipage}[t]{0.48\textwidth}
\begin{algorithm}[H]
\caption*{\textbf{(b) IndexCache Inference}}
\begin{algorithmic}[1]
\REQUIRE Input~$\mathbf{X}$, layers~$1 \ldots N$, pattern~$\mathbf{c}$
\FOR{$\ell = 1$ \TO $N$}
    \IF{\textcolor{red}{$c_\ell = \texttt{F}$}}
        \STATE $\mathbf{I}^{(\ell)} \leftarrow \textproc{Indexer}_\ell(\mathbf{X})$
        \STATE $\mathcal{T}^{(\ell)} \leftarrow \mathrm{Top\text{-}k}(\mathbf{I}^{(\ell)})$
        \STATE \textcolor{red}{$\mathcal{T}_{\mathrm{cache}} \leftarrow \mathcal{T}^{(\ell)}$}
    \ELSE[\textcolor{red}{$c_\ell = \texttt{S}$}]
        \STATE \textcolor{red}{$\mathcal{T}^{(\ell)} \leftarrow \mathcal{T}_{\mathrm{cache}}$} \hfill \textcolor{red}{$\triangleright$~reuse}
    \ENDIF
    \STATE $\mathbf{X} \leftarrow \textproc{SparseAttn}_\ell(\mathbf{X},\, \mathcal{T}^{(\ell)})$
    \STATE $\mathbf{X} \leftarrow \textproc{FFN}_\ell(\mathbf{X})$ \hfill $\triangleright$~\textrm{\small + norm, residual, etc.}
\ENDFOR
\end{algorithmic}
\end{algorithm}
\end{minipage}
\caption{%
    Side-by-side comparison of inference loops.
    \textbf{(a)}~Standard DSA runs the lightning indexer at every layer.
    \textbf{(b)}~IndexCache adds a single \textcolor{red}{conditional branch} (red lines): \texttt{F}~layers compute and cache fresh indices; \texttt{S}~layers reuse the cached indices.
    Note that~$\mathcal{T}_{\mathrm{cache}}$ is a temporary buffer holding only the current index tensor; it is overwritten at each~\texttt{F}~layer and requires no additional GPU memory beyond what standard DSA already allocates.
}
\label{fig:pseudocode-comparison}
\end{figure}

\section{Preliminary}
\label{sec:background}

\subsection{DeepSeek Sparse Attention}
\label{sec:bg-dsa}

DeepSeek Sparse Attention~(DSA)~\citep{deepseekv32} decomposes each attention layer into two stages: \emph{selection} and \emph{computation}.
A lightweight \emph{lightning indexer} first scores all preceding tokens against the current query using a multi-head ReLU-gated dot product, then selects the top-$k$ highest-scoring positions.
The main attention is computed only over this sparse subset, reducing per-layer core attention from~$O(L^2)$ to~$O(Lk)$ with~$k{=}2048 \ll L$, where~$L$ is the sequence length.
The indexer is designed for efficiency: it uses few heads, low-rank projections, and FP8~arithmetic, making it an order of magnitude cheaper per-FLOP than the main Multi-head Latent Attention~(MLA)~\citep{deepseekv2}.

DSA is instantiated under MLA and introduced through two-staged continue pre-training.
First, a short \emph{dense warm-up} trains only the indexer via KL-divergence distillation against the aggregated full attention distribution at each layer, while all other parameters are frozen.
Then, a longer \emph{sparse training} phase activates top-$k$ selection and jointly optimizes the entire model, with the indexer receiving distillation gradients on a detached computational graph.

Despite these efficiency gains, the indexer itself still operates at~$O(L^2)$: at every layer, it must independently score all preceding tokens to determine its own top-$k$ set.
Across a model with~$N$ layers, the total indexer cost is~$O(NL^2)$, and at long context lengths this becomes a significant fraction of the total attention budget.
A natural question is whether all~$N$ per-layer indexer computations are truly necessary, and whether the redundancy across layers can be exploited.

\subsection{Cross-Layer Stability of Token Selection}
\label{sec:bg-stability}

The answer comes from a broader empirical finding: the set of important tokens is remarkably stable across consecutive transformer layers.
Both \citet{kascade} and \citet{hysparse} observe that adjacent layers share the vast majority of their top-$k$ attention mass, and exploit this by designating a few \emph{anchor layers} that compute full attention while letting intermediate layers reuse the anchor's top-$k$ indices.

Crucially, both approaches depend on \emph{full attention} as the oracle for identifying important tokens.
In DSA, full attention has been eliminated entirely---replaced by the lightweight indexer.
This raises a question that, to our knowledge, has not been addressed: \emph{does the indexer's output also exhibit cross-layer stability?}
If so, we can apply the same sharing principle to eliminate redundant indexer computations without requiring any full attention oracle at all.
Furthermore, what is the maximum reuse ratio achievable before quality degrades, and can we adapt the model to close the performance gap introduced by aggressive index reuse? We present two novel complementary methods to achieve this in the following section.
\section{Method}
\label{sec:method}

\paragraph{Notation.}
We denote the number of transformer layers by~$N$, the sequence length by~$L$, and the number of selected tokens per query by~$k$.
At layer~$\ell$, the lightning indexer produces a score vector~$\mathbf{I}^{(\ell)}_{t} \in \mathbb{R}^{L}$ for query position~$t$, from which the top-$k$ index set~$\mathcal{T}^{(\ell)}_t = \mathrm{Top\text{-}k}(\mathbf{I}^{(\ell)}_t)$ is extracted ($k=2048$ throughout this paper).
We write~$\mathbf{p}^{(\ell)}_{t}$ for the aggregated attention distribution at layer~$\ell$ for position $t$ (obtained by averaging softmax attention weights across heads), and~$\mathbf{q}^{(\ell)}_t = \mathrm{Softmax}(\mathbf{I}^{(\ell)}_t)$ for the indexer's output distribution.

\paragraph{Overview.}
IndexCache modifies DSA by partitioning the~$N$ layers into two roles, encoded as a binary \emph{pattern string}~$\mathbf{c} = c_1 c_2 \cdots c_N$ with~$c_\ell \in \{\texttt{F}, \texttt{S}\}$:
\begin{itemize}[leftmargin=*,itemsep=2pt]
    \item \texttt{F}~(\textbf{Full}): the layer retains its indexer, computes fresh~$\mathcal{T}^{(\ell)}_t$ over all preceding tokens, and performs sparse core attention on the selected subset, identical to standard DSA.
    \item \texttt{S}~(\textbf{Shared}): the layer has \emph{no indexer}. It inherits the index set from the nearest preceding~\texttt{F}~layer, \emph{i.e.}\@\xspace, $\mathcal{T}^{(\ell)}_t \leftarrow \mathcal{T}^{(f(\ell))}_t$ where~$f(\ell) = \max\{j < \ell : c_j = \texttt{F}\}$, and directly applies sparse core attention using those inherited indices.
\end{itemize}
The first layer is always~\texttt{F} to seed the initial indices.
At inference, an~\texttt{S}~layer simply skips the indexer forward pass and reuses the cached index tensor from its~\texttt{F}~predecessor.
Figure~\ref{fig:pseudocode-comparison} illustrates the simplicity of this modification: the only change to the per-layer inference loop is a single conditional branch that either runs the indexer or copies the cached indices.

The key design question is how to choose the pattern~$\mathbf{c}$.
If most layers can safely share indices, a large fraction of the~$O(NL^2)$ total indexer cost can be eliminated while the~$O(NLk)$ core attention remains unchanged.
We propose two approaches: a \emph{training-free} method that determines~$\mathbf{c}$ via greedy search on an established DSA model~(Section~\ref{sec:method-trainfree}), and a \emph{training-aware} method that jointly optimizes the indexer parameters for cross-layer sharing via a multi-layer distillation loss~(Section~\ref{sec:method-trainaware}).

\subsection{Training-Free IndexCache}
\label{sec:method-trainfree}

Given a pretrained DSA model, our goal is to find a pattern~$\mathbf{c}$ that maximizes the number of~\texttt{S}~layers while minimizing the impact on model quality.
We first discuss why the most obvious approach fails, then present our greedy search algorithm.

\subsubsection{Why Uniform Interleaving Is Suboptimal}
\label{sec:method-interleave}

The simplest strategy is uniform interleaving: retain every~$r$-th layer's indexer and skip the rest (e.g., \texttt{FSSSFSSS}\dots for~$r{=}4$).
However, this ignores the fact that indexer importance varies significantly across layers.
We observe empirically that certain layers, particularly those in the early and transitional regions of the network, are far more sensitive to indexer removal than others.
Uniform interleaving may remove a critical indexer while retaining a redundant one, leading to noticeable quality degradation (see Section~\ref{sec:exp-trainfree} for quantitative comparisons).
This motivates a data-driven approach: let the model itself tell us which indexers are expendable, leading to our greedy layer selection algorithm.

\begin{algorithm}[t]
\caption{Greedy Layer Selection for Training-Free IndexCache}
\label{alg:greedy}
\begin{algorithmic}[1]
\REQUIRE DSA model~$M$ with~$N$ layers, calibration batches~$\mathcal{D}$, target \#~of~\texttt{S}~layers~$K$
\ENSURE Optimized pattern~$\mathbf{c}^*$
\STATE $\mathbf{c} \leftarrow \texttt{F}^N$ \hfill $\triangleright$~Initialize all layers as Full
\STATE $\mathcal{R} \leftarrow \{2, 3, \ldots, N\}$ \hfill $\triangleright$~Candidates (layer~1 always~\texttt{F})
\FOR{$\mathrm{step} = 1$ \TO $K$}
    \STATE $\ell^* \leftarrow \arg\min_{\ell \in \mathcal{R}} \;\textproc{EvalLoss}\bigl(M,\, \mathcal{D},\, \mathbf{c}|_{c_\ell \to \texttt{S}}\bigr)$
    \STATE $c_{\ell^*} \leftarrow \texttt{S}$,\; $\mathcal{R} \leftarrow \mathcal{R} \setminus \{\ell^*\}$
\ENDFOR
\RETURN $\mathbf{c}$
\end{algorithmic}
\end{algorithm}

\subsubsection{Layer Selection Algorithm}
\label{sec:method-greedy}

We propose a greedy search that incrementally converts~\texttt{F}~layers to~\texttt{S}~layers, using language modeling loss on a small calibration set as a proxy for downstream quality.

\paragraph{Calibration set.}
We cache~$B$ mini-batches from the training data.
All candidate patterns are evaluated on exactly the same batches, ensuring that loss differences reflect only the effect of the pattern change, not data variance. The loss is derived from a forward pass on the whole batch $\mathcal{D}$: $\textproc{EvalLoss}\bigl(M,\, \mathcal{D},\, \mathbf{c}\bigl)$.

\paragraph{Search procedure.}
Starting from the all-\texttt{F} baseline~($c_\ell = \texttt{F}$ for all~$\ell$), the algorithm proceeds for~$K$ steps, where~$K$ is the target number of~\texttt{S}~layers (e.g., $K = 3N/4$ to retain only~$1/4$ of indexers).
At each step, we iterate over all currently-\texttt{F} layers (excluding the first layer), tentatively flip each to~\texttt{S}, evaluate the resulting LM loss, and commit the flip that yields the lowest loss.
Algorithm~\ref{alg:greedy} presents the full procedure.

\paragraph{Complexity.}
A full search from all-\texttt{F} to all-\texttt{S} performs~$N(N{-}1)/2$ forward passes.
When pipeline parallelism partitions the model into~$P$ stages, we accelerate the search by splitting layers into~$P$ blocks (with each block's first layer fixed as~\texttt{F}) and searching blocks sequentially within each step: the best flip in each block is committed before the next block is searched, so that up to~$P$ layers are placed per step and total forward passes are reduced by roughly~$P{\times}$.

\begin{wrapfigure}{r}{0.35\linewidth}
    \centering
    \vspace{-1em}
    \includegraphics[width=\linewidth]{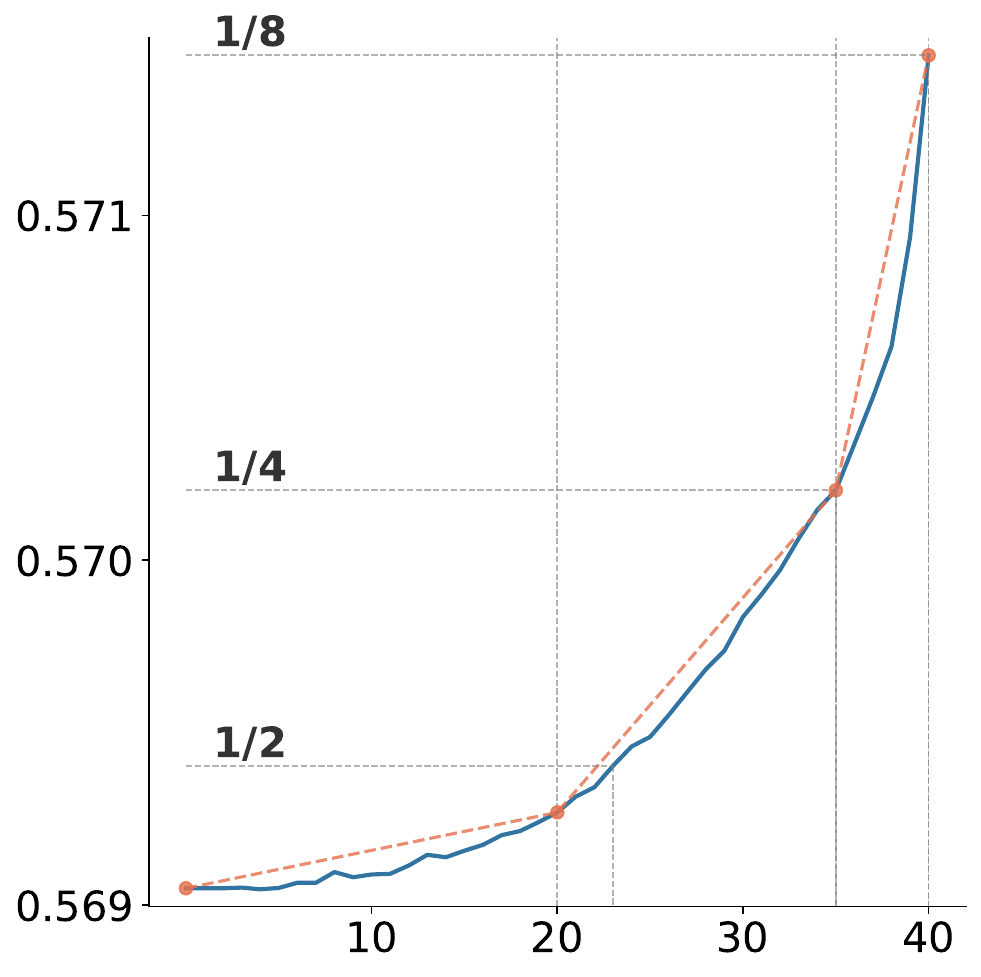}
    \vspace{-4em}
\end{wrapfigure}
\paragraph{Properties of the greedy solution.}
Although greedy search does not guarantee global optimality, we consistently observe three satisfying properties:
\begin{enumerate}[leftmargin=*,itemsep=1pt,label=(\arabic*)]
    \item The searched pattern outperforms uniform interleaving at the same retention ratio (Table~\ref{tab:trainfree}).
    \item As shown in the right figure (the steps for the 1/2, 1/4, and 1/8 retention ratios are marked), the per-step LM validation loss curve reveals a clear separation between ``easy'' layers (the first 20 steps) and ``critical'' layers (after 35 steps), suggesting a natural ordering of indexer importance.
    \item Results are stable across different calibration sets, indicating that this importance ranking is an intrinsic model property rather than a data artifact. Moreover, the LM loss serves as a valid proxy for downstream tasks, as lower LM loss is positively correlated with better task performance.
\end{enumerate}

\subsection{Training-Aware IndexCache with Multi-Layer Distillation}
\label{sec:method-trainaware}

Training-free IndexCache requires no weight updates, but is limited by the fact that each indexer was originally trained to serve only its own layer.
When training a DSA model from scratch or via continued pre-training, we can do better: explicitly training each retained indexer to serve \emph{multiple} layers simultaneously.

\paragraph{From single-layer to multi-layer distillation.}
In standard DSA training~(Section~\ref{sec:bg-dsa}), each indexer at layer~$\ell$ is distilled via KL divergence against its own layer's aggregated attention distribution~$\mathbf{p}^{(\ell)}_t$: $\mathcal{L}^{\mathrm{I}}
    = \sum_{t} D_{\mathrm{KL}}\!\left(
        \mathbf{p}^{(\ell)}_{t} \,\big\|\, \mathbf{q}^{(\ell)}_t
      \right)$.
We generalize this to a multi-layer objective.
Let layer~$\ell$ be a retained~\texttt{F}~layer, and let layers~$\ell{+}1, \ldots, \ell{+}m$ be the subsequent~\texttt{S}~layers that will reuse its index set~$\mathcal{T}^{(\ell)}_t$.
The multi-layer distillation loss is:
{\footnotesize
\begin{equation}
    \mathcal{L}^{\mathrm{I}}_{\mathrm{multi}}
    = \sum_{j=0}^{m} \frac{1}{m+1}\sum_{t}
      D_{\mathrm{KL}}\!\left(
        \mathbf{p}^{(\ell+j)}_{t} \,\big\|\, \mathbf{q}^{(\ell)}_t
      \right),
    \label{eq:multi-distill}
\end{equation}
}

Intuitively, this encourages the indexer to predict a top-$k$ set that is jointly useful for all layers it serves, rather than overfitting to layer~$\ell$ alone.

\paragraph{Gradient equivalence to distillation against the averaged distribution.}
A natural concern is whether optimizing a sum of KL terms introduces unexpected interactions.
We show that the multi-layer loss has a clean interpretation: it produces \emph{exactly the same gradient} as distilling against a single averaged target.

Define the averaged target~$\bar{\mathbf{p}}_{t} = \sum_{j=0}^{m} \frac{1}{m+1}\mathbf{p}^{(\ell+j)}_{t}$ and the corresponding single-target loss:
{\footnotesize
\begin{equation}
    \mathcal{L}^{\mathrm{I}}_{\mathrm{avg}}
    = \sum_{t} D_{\mathrm{KL}}\!\left(
        \bar{\mathbf{p}}_{t} \,\big\|\, \mathbf{q}^{(\ell)}_t
      \right).
    \label{eq:avg-distill}
\end{equation}
}

\begin{proposition}
\label{prop:grad-equiv}
$\nabla_\theta \mathcal{L}^{\mathrm{I}}_{\mathrm{multi}} = \nabla_\theta \mathcal{L}^{\mathrm{I}}_{\mathrm{avg}}$.
\end{proposition}

\begin{proof}
Since~$\mathbf{q}^{(\ell)}_t$ is the only parameter-dependent term in~$D_{\mathrm{KL}}(\mathbf{p} \| \mathbf{q}^{(\ell)}_t)$, the entropy of~$\mathbf{p}$ vanishes under differentiation:
$\nabla_\theta D_{\mathrm{KL}}(\mathbf{p} \| \mathbf{q}^{(\ell)}_t) = -\nabla_\theta \sum_{s} \mathbf{p}(s) \log \mathbf{q}^{(\ell)}_t(s)$.
Apply to Eq.~\ref{eq:multi-distill}:
{\footnotesize
\begin{align}
    \nabla_\theta \, \mathcal{L}^{\mathrm{I}}_{\mathrm{multi}}
    &= -\sum_{j=0}^{m} \frac{1}{m+1} \sum_{t} \nabla_\theta \sum_{s} \mathbf{p}^{(\ell+j)}_{t}(s) \log \mathbf{q}^{(\ell)}_t(s) \notag \\
    &= -\sum_{t} \nabla_\theta \sum_{s} \underbrace{\Bigl(\textstyle\sum_{j=0}^{m} \frac{1}{m+1} \mathbf{p}^{(\ell+j)}_{t}(s)\Bigr)}_{\bar{\mathbf{p}}_{t}(s)} \log \mathbf{q}^{(\ell)}_t(s)
    \;=\; \nabla_\theta \, \mathcal{L}^{\mathrm{I}}_{\mathrm{avg}}.
    \label{eq:grad-equiv-proof}
\end{align}
}
\end{proof}

\paragraph{Interpretation.}
Proposition~\ref{prop:grad-equiv} shows that multi-layer distillation is not merely a heuristic regularizer---it is \emph{exactly} equivalent to distilling the indexer toward the \emph{centroid} of the target layers' attention distributions.
The indexer therefore learns to predict a consensus top-$k$ that jointly covers the important tokens across all served layers.

Although the two loss formulations yield identical gradients, we adopt
$\mathcal{L}^{\mathrm{I}}_{\mathrm{multi}}$ in practice for implementation
efficiency.
When the subsequent layer is an \texttt{S} layer, it only needs to receive the current layer’s predicted $\mathbf{q}^{(\ell)}$. In contrast, training with $\mathcal{L}^{\mathrm{I}}_{\mathrm{avg}}$ requires passing both $\mathbf{q}^{(\ell)}$ and $\mathbf{p}^{(\ell)}$, which introduces unnecessary memory overhead and additional runtime cost.

\paragraph{Training.}
We follow the standard DSA training procedure with two stages. In the warm-up phase, we train the indexer in the \texttt{F} layer using $\mathcal{L}^{\mathrm{I}}_{\mathrm{multi}}$, while keeping all other parameters fixed. In the sparse training phase, we continue to train the indexer using $\mathcal{L}^{\mathrm{I}}_{\mathrm{multi}}$, defined as the KL divergence computed only over the selected top-$k$ tokens, and additionally include the LM loss to train the remaining parameters.

\section{Experiments}
\label{sec:experiments}

\subsection{Setup}
\label{sec:exp-setup}

\paragraph{Model.}
The DSA model used in our experiments was obtained through a two-stage training process starting from the base model of GLM-4.7-Flash\footnote{\url{https://huggingface.co/zai-org/GLM-4.7-Flash}}, a 30B-A3B MoE model with Multi-head Latent Attention~(MLA) and 47 layers. Its evaluation performance is comparable to that of the original GLM-4.7-Flash (see Table~6 in \citet{glm5}).

\paragraph{Training-free IndexCache.}
The greedy pattern search is guided by the per-token validation loss computed on SFT data with a batch size of 768 and a context length of 200K.

\paragraph{Training-aware IndexCache.}
A full DSA training pipeline starting from the base model requires substantial computational resources. Instead, we initialize directly from the GLM-4.7-Flash model and train it into a DSA model on SFT data with a context length of 200K. The training consists of a 1{,}000-step dense warm-up phase followed by a 4{,}000-step sparse training phase. This shorter pipeline closely matches the performance of full DSA training and suffices for evaluating IndexCache's training-aware component.

\paragraph{Evaluation.}
We include five long-context benchmarks: MRCR~v2~\citep{mrcr}, GraphWalks~\citep{graphwalks}, LongBench~v2~\citep{longbenchv2}, RULER~\citep{ruler}, and AA-LCR~\citep{aalcr}; four general \& reasoning benchmarks: AIME~2025~\citep{aime2025}, GPQA-Diamond~\citep{gpqa}, LiveCodeBench~v6~\citep{livecodebench}, and IFBench~\citep{ifbench}.
Evaluation setups are further detailed in Appendix~\ref{app:eval}.

\begin{table}[t]
\centering
\caption{%
    End-to-end inference performance of the 30B DSA model with IndexCache at two retention ratios.
    \textbf{Prefill time}: seconds (lower is better).
    \textbf{Decode per request}: tokens/s under single concurrency (higher is better).
    \textbf{Decode full}: total tokens/s (higher is better).
    Decode throughput is reported per GPU.
}
\label{tab:profiling}
\small
\begin{tabular}{lcccc}
\toprule
& 10K & 60K & 120K & 200K \\
\midrule
\multicolumn{5}{l}{\textit{Prefill time (s)} $\downarrow$} \\
\quad DSA & 0.57 & 3.38 & 8.57 & 19.5 \\
\quad + IndexCache (1/2) & 0.47 & 2.86 & 6.57 & 13.7 \\
\quad + IndexCache (1/4) & \textbf{0.45} & \textbf{2.59} & \textbf{5.66} & \textbf{10.7} \\
\midrule
\multicolumn{5}{l}{\textit{Decode throughput, per request (tok/s)} $\uparrow$} \\
\quad DSA & 73.5 & 67.0 & 63.0 & 58.0 \\
\quad + IndexCache (1/2) & 84.5 & 80.0 & 77.0 & 73.0 \\
\quad + IndexCache (1/4) & \textbf{91.0} & \textbf{89.5} & \textbf{88.0} & \textbf{86.0} \\
\midrule
\multicolumn{5}{l}{\textit{Decode throughput, full KV cache (tok/s)} $\uparrow$} \\
\quad DSA & 2700 & 613 & 341 & 197 \\
\quad + IndexCache (1/2) & 3070 & 750 & 431 & 253 \\
\quad + IndexCache (1/4) & \textbf{3310} & \textbf{840} & \textbf{498} & \textbf{297} \\
\bottomrule
\end{tabular}
\end{table}

\begin{figure}[t]
\centering
\includegraphics[width=\textwidth]{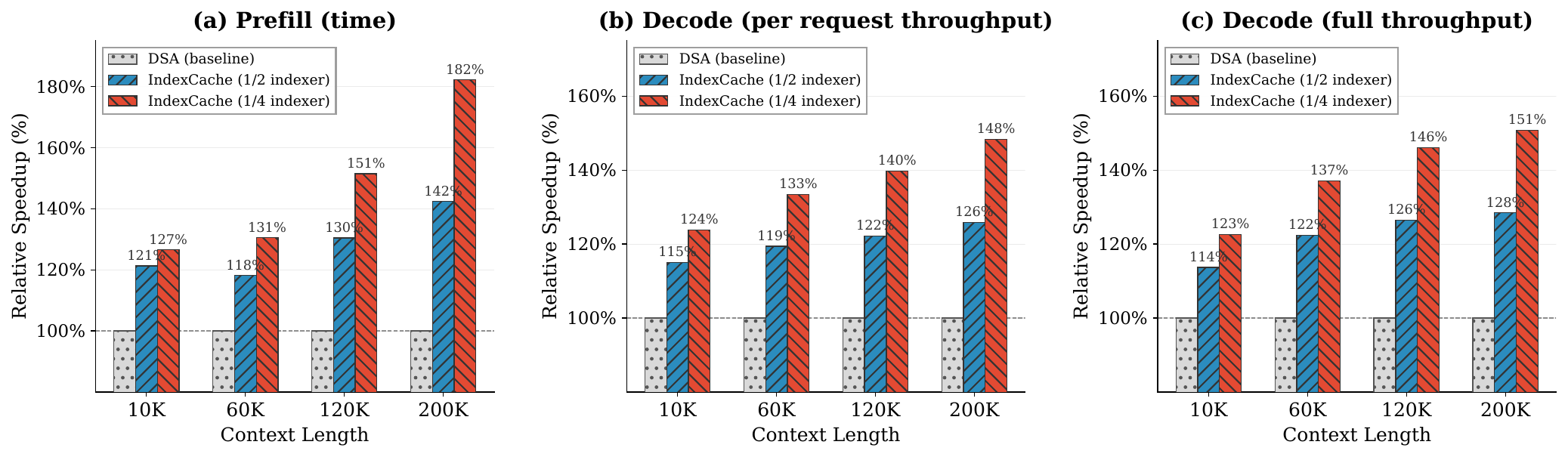}
\caption{%
    Relative speedup of IndexCache over the DSA baseline across three inference settings on the 30B model.
    DSA baseline is normalized to 100\%.
}
\label{fig:profiling-speedup}
\end{figure}

\subsection{End-to-End Inference Speedup}
\label{sec:exp-profiling}
We measure end-to-end inference performance using the 30B-parameter DSA model served with \texttt{dp\_attention} enabled (\texttt{dp\_size=8}) in SGLang, running on an NVIDIA H100 node.
We compare the original DSA baseline against IndexCache at two retention ratios: 1/2 (half of the indexer layers retained) and 1/4 (a quarter retained).
We report three complementary metrics across context lengths of 10K, 60K, 120K, and 200K tokens:
(1)~\emph{prefill latency} which measures time-to-first-token;
(2)~\emph{per-request decode throughput} under single concurrency (single request on each GPU);
and (3)~\emph{total decode throughput} when the KV cache is fully utilized (${\sim}$800K tokens on each GPU).
Table~\ref{tab:profiling} and Figure~\ref{fig:profiling-speedup} summarize the results.

\paragraph{Prefill.}
IndexCache delivers substantial prefill acceleration that grows with context length.
At 200K tokens, IndexCache~(1/4) reduces prefill latency from 19.5s to 10.7s, achieving a \textbf{1.82$\times$} speedup, by eliminating 75\% of the indexer computations that dominate the prefill phase.
Even at 10K, where the indexer accounts for a smaller fraction of total compute, a 1.27$\times$ speedup is observed.
Extrapolating to longer contexts ($>$200K), IndexCache is expected to deliver even greater speedups.

\paragraph{Decode.}
The per-request decode throughput improvement is significant at long contexts.
At 200K, DSA's decode speed is 58 tok/s, while IndexCache~(1/4) achieves 86 tok/s, a \textbf{1.48$\times$} speedup.
This is because the decode phase in DSA involves a per-token indexer pass over the full context, which becomes the bottleneck at long sequences; IndexCache directly reduces this bottleneck.
When the KV cache is fully saturated, IndexCache~(1/4) improves total decode throughput by 22-51\% across context lengths, with the largest gains at longer contexts (197$\to$297 tok/s at 200K, a \textbf{1.51$\times$} increase).

We observe similar trends on the larger GLM-5 model (744B parameters), where IndexCache~(1/4) yields at least \textbf{1.3$\times$} improvement in both prefill latency and decode throughput at context lengths beyond 100K.
Overall, our end-to-end prefill latency and decode throughput suggest IndexCache is particularly valuable for the long-context scenario.

\subsection{Training-Free IndexCache Results}
\label{sec:exp-trainfree}

Table~\ref{tab:trainfree} reports results of training-free IndexCache on the same 30B DSA model, comparing three retention ratios: 1/2, 1/4, and 1/8, each under a uniform interleave baseline and a greedy-searched pattern (Appendix~\ref{app:pattern}).

\begin{table}[t]
\centering
\caption{%
Training-free IndexCache at 1/2, 1/4, and 1/8 indexer retention.
`Long' and `G\&R' aggregate benchmark scores.
We compare uniform interleaving against searched patterns.
}
\label{tab:trainfree}
\resizebox{\textwidth}{!}{%
\begin{tabular}{lcc ccccc ccccc}
\toprule
& \multicolumn{2}{c}{\textbf{Averages}} & \multicolumn{5}{c}{\textbf{Long-Context}} & \multicolumn{4}{c}{\textbf{General \& Reasoning}} \\
\cmidrule(lr){2-3} \cmidrule(lr){4-8} \cmidrule(lr){9-12}
\textbf{Config} & Long & G\&R & MRCR & GW & LB2 & RULER & LCR & AIME & GPQA & LCB & IFB \\
\midrule
Original DSA & 50.2 & 74.6 & 24.5 & \textbf{49.6} & 45.5 & \textbf{87.9} & 43.6 & 91.0 & 77.6 & \textbf{71.4} & 58.4 \\
\midrule
\rowcolor{mygray}1/2 Unif. IndexCache & 47.4 & 74.3 & 22.0 & 46.6 & 46.0 & 83.6 & 38.6 & 92.2 & 76.4 & 69.7 & \textbf{59.0} \\
\quad+Search pattern & \textbf{50.3} & 74.4 & \textbf{24.7} & 49.5 & \textbf{46.3} & 87.8 & 43.2 & 91.9 & 76.3 & 71.3 & 58.2 \\
\rowcolor{mygray}1/4 Unif. IndexCache & 43.0 & 73.8 & 17.7 & 37.2 & 43.1 & 79.2 & 37.8 & 91.3 & 75.7 & 69.4 & 58.9 \\
\quad+Search pattern & 49.9 & \textbf{74.9} & 25.1 & 47.4 & 45.7 & 87.6 & \textbf{43.8} & \textbf{92.6} & \textbf{78.6} & 70.0 & 58.3 \\
\rowcolor{mygray}1/8 Unif. IndexCache & 35.3 & 70.0 & 12.9 & 33.1 & 37.7 & 68.8 & 24.0 & 89.1 & 74.1 & 58.7 & 58.0 \\
\quad+Search pattern & 46.1 & 73.7 & 21.7 & 43.8 & 42.3 & 82.0 & 40.8 & 90.7 & 76.5 & 69.6 & 58.1 \\
\bottomrule
\end{tabular}%
}
\end{table}

\paragraph{Searched patterns close the gap on long-context tasks.}
Uniform interleaving at aggressive retention ratios incurs significant long-context degradation: 1/2 and 1/4 uniform interleaving drops Long~Avg by 2.8 and 7.2 points (50.2$\to$47.4 and 50.2$\to$43.0)
The greedy-searched pattern largely eliminates this deficit, recovering Long~Avg to 49.9 at 1/4 retention and to 50.3 at 1/2 retention, both comparable to the original DSA.
This confirms that \emph{which} indexer layers are retained matters far more than how many.
Nevertheless, when retaining only 1/8 of the indexer layers, we observe a substantial degradation: Long~Avg drops to 35.3 with uniform interleaving and to 46.1 with the searched pattern. While search-based IndexCache still markedly mitigates the loss caused by removing indexers, the resulting decline in long-context performance at this extreme sparsity becomes non-negligible.

\paragraph{Long chain-of-thought reasoning capabilities are preserved.}
Across all configurations except for uniform interleave at 1/8 ratio, G\&R~Avg stays within 1 point of the 74.6 baseline (73.7-74.9 vs.\ 74.6).
Notably, the 1/4 searched pattern \emph{improves} over DSA on AIME~2025 (92.6 vs.\ 91.0) and GPQA-Diamond (78.6 vs.\ 77.6), suggesting that removing redundant indexer computation may act as a mild regularizer during inference.
This confirms that IndexCache does not trade general reasoning ability for long-context efficiency.

\subsection{Training-Aware IndexCache Results}
\label{sec:exp-trainaware}

We perform training-aware IndexCache using the multi-layer distillation loss~(Section~\ref{sec:method-trainaware}) at two retention ratios: 1/2 and 1/4, both with uniform interleaving.
We further ablate two design choices at the 1/2 retention ratio: (1) replacing the uniform interleaving pattern with the greedy-searched pattern from Section~\ref{sec:exp-trainfree}, and (2) removing the cross-layer distillation loss (i.e., training each indexer only against its own layer's attention distribution).
Table~\ref{tab:trainaware} reports the evaluation results.
Note that the DSA baseline in this section is also trained using our shortened DSA training pipeline, which results in a small performance difference compared to the original DSA reported in Table~\ref{tab:trainfree}. 

\begin{table}[t]
\centering
\caption{%
    Training-aware IndexCache at 1/2 and 1/4 indexer retention with uniform interleaving.
    \emph{w/ searched pattern}: the greedy-searched pattern replaces uniform interleaving.
    \emph{w/o cross-layer loss}: each indexer is distilled only against its own layer.
}
\label{tab:trainaware}
\resizebox{\textwidth}{!}{%
\begin{tabular}{lcc ccccc cccc}
\toprule
& \multicolumn{2}{c}{\textbf{Averages}} & \multicolumn{5}{c}{\textbf{Long-Context}} & \multicolumn{4}{c}{\textbf{General \& Reasoning}} \\
\cmidrule(lr){2-3} \cmidrule(lr){4-8} \cmidrule(lr){9-12}
\textbf{Config} & Long & G\&R & MRCR & GW & LB2 & RULER & LCR & AIME & GPQA & LCB & IFB \\
\midrule
Original DSA & 51.0 & 74.2 & \textbf{24.7} & 49.1 & 46.9 & 87.3 & 47.0 & 88.8 & \textbf{79.4} & 70.5 & 57.9 \\
\midrule
1/2 Unif. IndexCache & \textbf{51.6} & \textbf{74.5} & 23.8 & \textbf{50.2} & \textbf{47.2} & 87.0 & \textbf{49.8} & 89.3 & 76.7 & \textbf{72.2} & \textbf{59.9} \\
\rowcolor{mygray}\quad w/ searched pattern & 50.6 & 73.6 & 23.9 & 48.1 & 47.1 & \textbf{87.5} & 46.6 & \textbf{89.6} & 78.6 & 68.5 & 57.7 \\
\rowcolor{mygray}\quad w/o cross-layer loss & 49.8 & 74.5 & 24.6 & 48.3 & 45.0 & 87.1 & 44.0 & 88.8 & 79.4 & 71.7 & 58.0 \\
1/4 Unif. IndexCache & 50.6 & 74.1 & 23.7 & 48.1 & 46.9 & 86.1 & 48.4 & 89.3 & 78.0 & 70.5 & 58.7 \\
\bottomrule
\end{tabular}%
}
\end{table}

\paragraph{Training-aware IndexCache matches DSA baseline.}
Uniform IndexCache with 1/2 ratio achieves a Long~Avg of 51.6, \emph{surpassing} the baseline (51.0), while G\&R~Avg remains comparable (74.5 vs.\ 74.2).
At 1/4 retention, both Long~Avg and G\&R~Avg are within 0.4\% of the baseline.
These results confirm that DSA can be trained to adapt to the sharing pattern.

\paragraph{The pattern sensitivity observed in training-free IndexCache vanishes with training.}
A striking contrast with Section~\ref{sec:exp-trainfree} emerges: uniform interleaving at 1/2 retention performs on par with and even slightly above the greedy-searched pattern (Long~Avg 51.6 vs.\ 50.6; G\&R~Avg 74.5 vs.\ 73.6).
Recall that in the training-free setting, the searched pattern was essential for recovering quality at aggressive retention ratios.
This is because, without retraining, certain layers are strongly coupled to their own indexer's top-$k$ selection; inheriting indices from an earlier layer introduces a distributional shift that causes sharp performance drops.
The greedy search in the training-free setting works precisely by avoiding these sensitive layers.
However, when the model is retrained with a sharing-aware objective, the \texttt{S}~layers learn to adapt their attention to inherited indices, and the retained indexers simultaneously learn to produce selections that generalize across their served layers.
This joint adaptation eliminates the layer-specific sensitivity entirely, allowing even a simple uniform pattern to match the full-indexer baseline.

\paragraph{Cross-layer distillation provides a meaningful benefit.}
Removing the cross-layer loss drops Long~Avg from 51.6 to 49.8, with AA-LCR falling from 49.8 to 44.0.
This confirms that the multi-layer distillation objective is practically beneficial: by training each indexer toward the centroid of its served layers' attention distributions, it learns a consensus top-$k$ that generalizes across layers rather than overfitting to a single one.

\subsection{Scaling Experiment}
\label{sec:scales}

\begin{table}[t]
\centering
\caption{%
    Preliminary results on GLM-5 (744B) with training-free IndexCache.
}
\label{tab:scaling}
\resizebox{0.95\textwidth}{!}{%
\begin{tabular}{lc ccccc}
\toprule
& \textbf{Long Avg} & \textbf{MRCR v2} & \textbf{GraphWalks} & \textbf{LongBench v2} & \textbf{RULER} & \textbf{AA-LCR} \\
\midrule
Original DSA & 78.4 & 71.1 & \textbf{92.7} & 64.5 & \textbf{97.7} & 66.2 \\
\midrule
\rowcolor{mygray}1/2 Unif. IndexCache & 78.1 & \textbf{72.8} & 90.2 & 65.1 & 97.6 & 64.6 \\
\quad+Searched pattern & \textbf{78.7} & 72.3 & 90.8 & \textbf{66.0} & 97.3 & 67.2 \\
\rowcolor{mygray}1/4 Unif. IndexCache & 72.7 & 65.8 & 74.9 & 62.2 & 96.2 & 64.6 \\
\quad+Searched pattern & 78.0 & 70.8 & 90.3 & 63.7 & 97.6 & \textbf{67.6} \\
\bottomrule
\end{tabular}%
}
\end{table}

We apply training-free IndexCache to GLM-5, a 744B-parameter (40B active) model that uses DSA by default. Table~\ref{tab:scaling} reports results on five long-context benchmarks.
The overall trends mirror the 30B findings: uniform interleaving degrades at aggressive retention, while the searched pattern recovers quality.
Interestingly, uniform interleaving at 1/2 retention happens to preserve Long~Avg (78.1 vs.\ 78.4), but this is likely coincidental, where the fixed alternating pattern simply avoids skipping the most critical indexer layers by chance.
The searched pattern provides consistently stable results: at 1/2 retention it slightly \emph{exceeds} the baseline (78.7 vs.\ 78.4), and at 1/4 retention it remains within 0.4 points (78.0 vs.\ 78.4).
We also conducted an all-round evaluation of IndexCache with 1/2 indexer retention across all tests on the Artificial Analysis Index, and its performance is nearly identical to that of the original GLM‑5 model (Figure~\ref{fig:glm5}).

We plan to apply training-aware IndexCache to this production-scale model in the near future. Given that the training-free variant already matches baseline quality, we are optimistic that training-aware adaptation will further solidify these gains and translate into concrete deployment-time efficiency benefits.

\section{Related Work}
\label{sec:related}

\subsection{Efficient Attention}
\label{sec:related-efficient}

Reducing the quadratic cost of self-attention is a central research theme, especially in the era of long-horizon agents.
\emph{Training-free sparse} methods introduce sparsity at inference via fixed patterns, heuristic eviction strategies, or lightweight importance estimation~\citep{h2o,streamingllm,sparq,quest,minference,duoattention,xattention,flexprefill,moa,sampleattention,lserve,spargeattention}.
Nevertheless, the resulting training-inference mismatch can cause error accumulation in long-context settings~\citep{lil}.
In contrast, \emph{trainable sparse} methods incorporate sparsity directly into the training stage, for example, through learned gating mechanisms~\citep{seerattention,seerattentionr}, end-to-end sparse pre-training~\citep{nsa}, block-level mixture routing~\citep{moba,infllm-v2,minicpm4}, or full-to-sparse distillation~\citep{deepseekv32,ssa}.
DSA~\citep{deepseekv32}, the foundation of our work, distills a lightweight lightning indexer from full attention to select the top-$k$ tokens for each query, reducing the core attention complexity to~$O(Lk)$.
Beyond sparsity, \emph{hybrid architectures} reduce the number of expensive quadratic layers by interleaving them with sliding window attention (with or without sink)~\citep{gptoss,gemma3,mimov2,longcat}, linear attention~\citep{gdn,minimax01,nemotron-h,kimilinear}, or state-space layers~\citep{mamba,mamba2,jamba}.

\subsection{Cross-Layer Sharing}
\label{sec:related-crosslayer}
Recent studies demonstrate that representation exhibits strong consistency across adjacent layers.
This structural property is often leveraged to reduce computational redundancy and accelerate inference.
TidalDecode~\citep{tidaldecode}, LessIsMore~\citep{lessismore}, OmniKV~\citep{omnikv}, and DELTA~\citep{delta} reuse top-$k$ indices from periodic anchor layers for sparse decoding.
Kascade~\citep{kascade} formalizes anchor layer selection via dynamic programming over a cross-layer similarity matrix and identifies head-aware remapping as critical for maintaining accuracy.
All of these methods rely on \emph{full attention} at anchor layers to compute exact top-$k$ indices.
Independently, cross-layer \emph{KV cache sharing} reduces memory by letting multiple layers reuse the same key-value tensors~\citep{yoco,cla,minicache,swiftkv,mlkv,wu2025improving}.
HySparse~\citep{hysparse} unifies both directions, interleaving full attention layers with sparse layers that inherit both top-$k$ block indices and KV~caches.
However, all of these approaches require full attention layers as the oracle, which DSA removes completely.

IndexCache differs in two key aspects.
First, the oracle is fundamentally cheaper because we share the output of DSA's lightweight indexer rather than full~$O(L^2)$ attention scores.
Second, we introduce systematic techniques for optimizing the sharing configuration, including a training-free greedy search to identify the optimal structural layout and a training-aware multi-layer distillation loss for parameter adaptation.
Although we instantiate IndexCache on DSA, the core principle extends to any sparse attention method that does not rely on a fixed sparse pattern but rather involves a dynamic token selection step: for instance, the block-level selection in MoBA~\citep{moba} and NSA~\citep{nsa} could similarly benefit from cross-layer reuse.
\section{Conclusion}
\label{sec:conclusion}

We have presented IndexCache, a method that accelerates sparse attention by exploiting the cross-layer redundancy of the indexer in charge of token selection.
IndexCache partitions layers into a small number of~\texttt{F}~layers that retain their indexers and a majority of~\texttt{S}~layers that reuse inherited top-$k$ indices, eliminating up to~75\%\ of the~$O(NL^2)$ total indexer cost with a single conditional branch without any performance degradation.
More broadly, our work demonstrates that the cross-layer sharing principle that previously applied only where full attention serves as the oracle extends naturally to sparse attention.
As sparse attention becomes the default for frontier LLMs (DeepSeek-V3.2, GLM-5), we expect cross-layer index reuse to become a standard component of efficient inference pipelines.

\bibliography{references}
\bibliographystyle{colm2026_conference}

\newpage
\appendix
\section*{Appendix}
\label{sec:appendix}

\section{Cross-Layer Top-$k$ Index Overlap}
\label{app:topk-overlap}

To empirically validate the cross-layer redundancy of top-$k$ index selections in DSA, we compute the pairwise overlap ratio between the top-$k$ indices selected by each layer's lightning indexer.
Specifically, for every pair of layers~$(i, j)$, we measure~$|\mathcal{T}^{(i)} \cap \mathcal{T}^{(j)}| / k$ (where~$k = 2048$) averaged over 768 samples of 200K length in a calibration set.

\begin{figure}[htbp]
\centering
\includegraphics[width=0.7\linewidth]{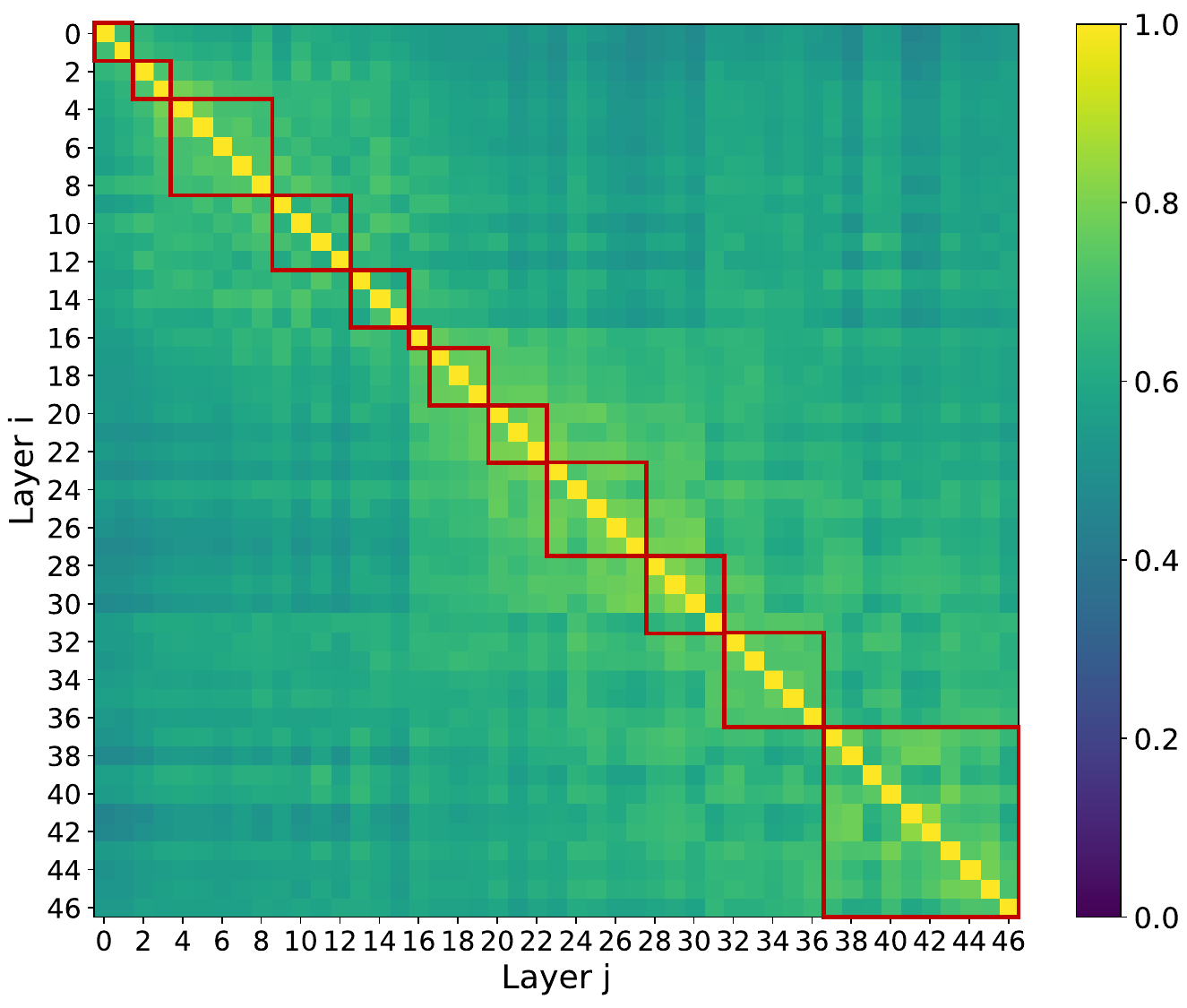}
\caption{Pairwise top-$k$ index overlap ratio between all layer pairs of the 30B DSA model. Shared blocks according to the greedy-searched 1/4 IndexCache pattern are \textcolor{BrickRed}{\textbf{marked}}.}
\label{fig:topk-iou-heatmap}
\end{figure}

Figure~\ref{fig:topk-iou-heatmap} visualizes this overlap as a heatmap for the 47-layer 30B DSA model, with the sharing blocks from the greedy-searched 1/4-retention pattern marked as red boxes.
Several patterns are evident:
\begin{itemize}[leftmargin=*,itemsep=2pt]
    \item \textbf{High overlap near the diagonal.} Adjacent layers exhibit overlap ratios of~0.7-1.0, confirming that consecutive layers select largely the same set of tokens.
    \item \textbf{Block structure.} The heatmap reveals distinct clusters of layers with mutually high overlap, for instance, layers~3-5, 6-8, 17-30, 31-36, etc., suggesting that the model organizes into functional blocks where token selection is internally consistent.
    \item \textbf{Uneven decay.} Overlap decreases more rapidly across block boundaries than within them, indicating that a few ``transition'' layers shift the attention focus substantially.
    \item \textbf{Early-late distinction.} The bottom-left and top-right corners of the heatmap are notably dark (overlap~$\leq 0.4$), showing that early and late layers attend to fundamentally different token subsets.
\end{itemize}

\paragraph{Greedy-searched blocks vs.\ overlap clusters.}
Comparing the red boxes (greedy-searched sharing blocks) with the natural overlap clusters reveals an informative mismatch.
While the greedy search does place~\texttt{F}~layers near some visually obvious cluster boundaries, the two partitions do \emph{not} fully coincide.
The root cause is that overlap is an \emph{aggregate} metric: it counts how many tokens are shared but not \emph{which} ones differ.
In the training-free setting, weights are frozen, so even a small set of mismatched critical tokens can perturb a layer's hidden state in ways that cascade through all downstream layers.
Early layers are especially vulnerable because their perturbations traverse the longest propagation path.

This observation also echoes our negative result with the similarity-based pattern search (Appendix~\ref{app:similarity-search}): local metrics---whether cosine similarity of attention outputs or top-$k$ index overlap---lack the discriminative power to identify the optimal sharing pattern, necessitating end-to-end evaluation.

\section{Searched Patterns}
\label{app:pattern}
Searched patterns for GLM-4.7-Flash 30B DSA:
\begin{itemize}[leftmargin=*,itemsep=2pt]
    \item Keep 1/2 `\texttt{F}'s: \texttt{FSFSFSSSSFSFFFFSFFSSFFSFFFSSFFSSFSSSSFSFFFSFSSF}
    \item Keep 1/4 `\texttt{F}'s: \texttt{FSFSFSSSSFSSSFSSFFSSFSSFSSSSFSSSFSSSSFSSSSSSSSS}
    \item Keep 1/8 `\texttt{F}'s: 
    \texttt{FSSSFSSSSSSSSFSSSFSSSSSFSSSSFSSSSSSSSFSSSSSSSSS}
\end{itemize}
Searched patterns for GLM-5:
\begin{itemize}[leftmargin=*,itemsep=2pt]
    \item Keep 1/2 `\texttt{F}'s: \texttt{FFSFSSSFSSFFFSSSFFFSFSSSSSSFFSFFSFFSSFFFFFFSFFFFFSFFSSSSSSFSFFFS\allowbreak FSSSFSFFSFFSSS}
    \item Keep 1/4 `\texttt{F}'s: \texttt{FFSFSSSFSSFSFSSSSSSSFSSSSSSFSSSFSFSSSSFFFFFSSSFFSSSFSSSSSSSSFSSS\allowbreak FSSSSSSFSFSSSS}
\end{itemize}

\section{Similarity-based Pattern Search}
\label{app:similarity-search}
We believe papers should report not only positive results but also negative (or unsuccessful) ones.
Before arriving at the greedy loss-based search described in Section~\ref{sec:method-greedy}, we explored a seemingly natural alternative: choosing the sharing pattern by directly measuring how similar the attention outputs are when an indexer is reused across layers.
Although this approach is theoretically motivated and computationally cheaper than the greedy search, it ultimately proved insufficient as a proxy for downstream quality.
We describe it here for completeness and to provide insight into why the loss-based search is necessary.

\paragraph{Constructing the similarity matrix.}
Given a DSA model with~$N$ layers, we perform $N$ single forward passes over a calibration set and, for each layer pair~$(i, j)$ with~$i > j$, compute the cosine similarity between:
\begin{enumerate}[leftmargin=*,itemsep=1pt]
    \item the core attention output at layer~$i$ when using layer~$i$'s \emph{own} indexer (i.e., the original model), and
    \item the core attention output at layer~$i$ when \emph{reusing} layer~$j$'s indexer (i.e., as if layer~$i$ were an~\texttt{S}~layer sharing from layer~$j$).
\end{enumerate}
This yields an $N \times N$ lower-triangular similarity matrix~$\mathbf{S}$, where $S_{i,j}$ quantifies how well layer~$j$'s index can serve as a proxy for layer~$i$'s own index.
Intuitively, if $S_{i,j}$ is close to~1, then layer~$i$ can safely skip its own indexer computation and reuse the index from layer~$j$ with minimal distortion to its attention output.

\paragraph{Dynamic programming formulation.}
Given the similarity matrix~$\mathbf{S}$ and a target number of~\texttt{F}~layers~$M$ (equivalently, $N - M$~\texttt{S}~layers), we seek the pattern~$\mathbf{c}^*$ that maximizes the total similarity:
\begin{equation}
    \mathbf{c}^* = \arg\max_{\mathbf{c}:\, |\{i : c_i = \texttt{F}\}| = M} \;\sum_{\ell:\, c_\ell = \texttt{S}} S_{\ell,\, \mathrm{src}(\ell)}
\end{equation}
where $\mathrm{src}(\ell)$ denotes the most recent preceding~\texttt{F}~layer from which layer~$\ell$ inherits its index (i.e., $\mathrm{src}(\ell) = \max\{j < \ell : c_j = \texttt{F}\}$).

This can be solved exactly via dynamic programming.
Let $\mathrm{dp}[i][k]$ denote the maximum cumulative similarity achievable for layers~$1, \ldots, i$ using exactly~$k$~\texttt{F}~layers, with layer~$i$ itself being an~\texttt{F}~layer.
The transition considers all possible previous~\texttt{F}~layers:
\begin{equation}
    \mathrm{dp}[i][k] = \max_{j < i,\; c_j = \texttt{F}} \left\{ \mathrm{dp}[j][k{-}1] + \sum_{m=j+1}^{i-1} S_{m,j} \right\}
\end{equation}
where the summation accounts for all~\texttt{S}~layers between~$j$ and~$i$, each reusing layer~$j$'s index.
We recover the optimal pattern by backtracking through the DP table.

\begin{table}[t]
\centering
\caption{%
    Evaluation results of training-free \emph{similarity-based} searched pattern.
}
\label{tab:sim}
\resizebox{0.7\textwidth}{!}{%
\begin{tabular}{l cccc}
\toprule
& \textbf{Avg} & \textbf{MRCR v2} & \textbf{GraphWalks} & \textbf{RULER} \\
\midrule
Original DSA & \textbf{54.0} & \textbf{24.5} & \textbf{49.6} & \textbf{87.9}  \\
\midrule
\rowcolor{mygray}1/2 Unif. IndexCache & 50.7 & 22.0 & 46.6 & 83.6 \\
\quad+Searched pattern & 49.8 & 22.9 & 43.5 & 82.9 \\
\bottomrule
\end{tabular}%
}
\end{table}

\paragraph{Result: similarity-optimal patterns perform comparably to uniform interleaving on downstream tasks.}
Like uniform interleaving, DP-searched patterns exhibit the same significant quality degradation relative to the original DSA model, as shown in Table~\ref{tab:sim}.
In other words, despite explicitly optimizing for maximal cross-layer similarity, the resulting patterns offer no meaningful advantage over the na\"{\i}ve uniform baseline.
In contrast, the greedy loss-based search (Section~\ref{sec:method-greedy}) produces patterns that \emph{substantially} outperform both uniform and similarity-optimal patterns, especially at aggressive retention ratios.

\paragraph{Why similarity fails as a proxy.}
The fundamental issue is that per-layer output similarity is a \emph{local} metric: it measures how well a single layer's attention output is preserved in isolation, without accounting for how small perturbations propagate across the remaining layers.
Two layers may have nearly identical attention outputs ($S_{i,j} \approx 1$) yet differ in subtle ways that matter for downstream quality: for instance, the reused index may miss a small number of critical tokens whose importance only becomes apparent in later layers' reasoning steps.
These subtle mismatches accumulate through the layers, leading to non-negligible final quality degradation that a layer-local similarity score cannot predict.

The greedy loss-based search avoids this pitfall by directly optimizing a \emph{global} metric, i.e., the LM loss, which captures the end-to-end effect of each sharing decision on the model's output distribution.
This allows it to identify ``critical'' layers (those whose indexers must be retained to avoid cascading errors) that the similarity-based approach treats as interchangeable with their neighbors.

\section{Evaluation Setup}
\label{app:eval}
All benchmarks are evaluated with temperature~1.0, top-$p$~=~0.95, and top-$k$~=~40. For long-context tasks, we set a total context window of 200K tokens with 32K reserved for the output. For general \& reasoning tasks, we allow a maximum output length of 64K tokens.
On MRCR~v2, we report the average score across 2-, 4-, and 8-needle settings.
On GraphWalks, we report the average score over the Parent-type and BFS-type problems.
On RULER, we report scores on all instances with context lengths ranging from 4K to 128K.
For MRCR~v2 and GraphWalks, we include only instances whose input length fits within the effective input budget (200K$-$32K$=$168K tokens).
For LongBench~v2, RULER, and AA-LCR, we include all instances and apply middle truncation to those exceeding 168K tokens.

\end{document}